\documentclass[sigconf]{acmart}

\AtBeginDocument{%
  \providecommand\BibTeX{{%
    \normalfont B\kern-0.5em{\scshape i\kern-0.25em b}\kern-0.8em\TeX}}}

\usepackage{amsmath}
\usepackage{sansmath}
\usepackage{bm}
\usepackage{float}
\usepackage{graphicx}
\usepackage{float} 
\usepackage{subfigure}
\usepackage{amsfonts}
\usepackage{amsthm}
\usepackage{graphicx}
\usepackage{xcolor}
\usepackage{stfloats}
\usepackage{lipsum}
\usepackage{fullpage}
\makeatletter
\newtheorem*{rep@theorem}{\rep@title}
\newcommand{\newreptheorem}[2]{%
\newenvironment{rep#1}[1]{%
 \def\rep@title{#2 \ref{##1}}%
 \begin{rep@theorem}}%
 {\end{rep@theorem}}}
\makeatother
\newtheorem{theorem}{Theorem}
\newtheorem{definition}[theorem]{Definition}

\newtheorem{lemma}[theorem]{Lemma}

\newtheoremstyle{named}{}{}{\itshape}{}{\bfseries}{.}{.5em}{\thmnote{#3}}
\theoremstyle{named}
\newtheorem*{namedtheorem}{Theorem}

\usepackage[linesnumbered,ruled,vlined]{algorithm2e}

\SetCommentSty{mycommfont}

\SetKwInput{KwInput}{Input}                
\SetKwInput{KwOutput}{Output}              

\DeclareMathOperator*{\argmax}{arg\,max}
\DeclareMathOperator*{\argmin}{arg\,min}
\usepackage{natbib} 
\usepackage{mathtools} 
\usepackage{tikz} 

\DeclareSymbolFont{extraup}{U}{zavm}{m}{n}
\DeclareMathSymbol{\varheart}{\mathalpha}{extraup}{86}
\DeclareMathSymbol{\vardiamond}{\mathalpha}{extraup}{87}

\makeatletter
\newcommand*{\rom}[1]{\expandafter\@slowromancap\romannumeral #1@}
\usepackage{booktabs}
\usepackage{tabularx}
\newcolumntype{Y}{>{\centering\arraybackslash}X}
\allowdisplaybreaks
\setcopyright{acmcopyright}
\copyrightyear{2021}
\acmYear{2021}

\setcopyright{none}

\begin{document}
\fancyhead{} 
\acmConference[WSDM '22]{The fifteenth ACM International Conference on Web Search and Data Mining (WSDM ’22)}{February  21--25, 2022}{Phoenix, Arizona}
\acmBooktitle{The fifteenth ACM International Conference on Web Search and Data Mining (WSDM ’22), February  21--25, 2022, Phoenix, Arizona}
\title{Combinatorial Bandits under Strategic Manipulations}



\author{Jing Dong}
\email{jingdong@link.cuhk.edu.cn}
\affiliation{%
  \institution{The Chinese University of Hong Kong, Shenzhen}\country{}
}

\author{Ke Li}
\email{keli@link.cuhk.edu.cn}
\affiliation{%
  \institution{The Chinese University of Hong Kong, Shenzhen}\country{}
}

\author{Shuai Li}
\email{shuaili8@sjtu.edu.cn}
\affiliation{%
  \institution{Shanghai Jiao Tong University}\country{}
}

\author{Baoxiang Wang}
\email{bxiangwang@cuhk.edu.cn}
\affiliation{%
  \institution{The Chinese University of Hong Kong, Shenzhen}
  \country{}
}

\begin{abstract}
Strategic behavior against sequential learning methods, such as ``click framing'' in real recommendation systems, have been widely observed. Motivated by such behavior we study the problem of combinatorial multi-armed bandits (CMAB) under strategic manipulations of rewards, where each arm can modify the emitted reward signals for its own interest. This characterization of the adversarial behavior is a relaxation of previously well-studied settings such as adversarial attacks and adversarial corruption. We propose a strategic variant of the combinatorial UCB algorithm, which has a regret of at most $O(m\log T + m B_{max})$ under strategic manipulations, where $T$ is the time horizon, $m$ is the number of arms, and $B_{max}$ is the maximum budget of an arm. We provide lower bounds on the budget for arms to incur certain regret of the bandit algorithm. Extensive experiments on online worker selection for crowdsourcing systems, online influence maximization and online recommendations with both synthetic and real datasets corroborate our theoretical findings on robustness and regret bounds, in a variety of regimes of manipulation budgets.
\end{abstract}




\maketitle

\section{Introduction}
\label{sec:intro}
Sequential learning methods feature prominently in a range of real applications such as online recommendation systems, crowdsourcing systems, and online influence maximization problems. Among those methods, the multi-armed bandits problem serves as a fundamental framework. Its simple yet powerful model characterizes the dilemma of exploration and exploitation which is critical to the understanding of online sequential learning problems and their applications \citep{10.2307/2332286, Robbins:1952,auer2002finite,lattimore_szepesvari_2020}.
The model describes an iterative game constituted by a bandit algorithm and many arms. The bandit algorithm is required to, through a horizon $T$, choose an arm to pull at each time step. As the objective is to maximize the cumulative reward over time, the algorithm balances between exploiting immediate rewards based on the information collected or pulling less explored arms to gain more information about arms \citep{anantharam_asymptotically_1987,auer_nonstochastic_2002,cesa-bianchi_prediction_2006}.

Out of the real applications, many motivate the extension of MAB towards combinatorial multi-armed bandits (CMAB), where multiple arms can be selected in each round \citep{pmlr-v28-chen13a,DBLP:conf/nips/CombesSPL15,pmlr-v48-lif16,pmlr-v97-zimmert19a,pmlr-v117-rejwan20a}. CMAB demonstrates its effectiveness on problems like online social influence maximization, viral marketing, and advertisement placement, within which many offline variants are NP-hard. However, existing MAB and CMAB algorithms are often developed either under benign assumptions on the arms \citep{pmlr-v28-chen13a,pmlr-v38-kveton15,pmlr-v75-wei18a} or purely adversarial arms \citep{auer2002nonstochastic}. In the former setting, the arms are commonly assumed to report their reward signals truthfully without any strategic behavior, under which the drawbacks are apparent. In the latter setting, arms can attack any deployed algorithm to regret of $O(T)$ with this capability of reward manipulations, which is catastrophic for a bandit algorithm. 
This assumption is stringent and rarely realistic.

In this paper, we adapt the combinatorial UCB (CUCB) algorithm with a carefully designed UCB-based exploration term. 
A major difficulty stems from not knowing the manipulation term made by the arms, while our algorithm overcomes this by depending only on the knowledge of the maximum possible strategic budget.
Previous results by \citep{pmlr-v119-feng20c} only implies robustness of UCB style algorithm under stochastic multi armed bandits setting under only full knowledge of step wise deployment of strategic budget.
New tail bounds over the proposed exploration term and a new trade-off parameter that balances exploration and exploitation are utilized to facilitate the analysis of our algorithm.
We further establish results on the robustness of our UCB variant under strategic arms with an algorithm-dependent budget lower bound.

Our proposed algorithms are also evaluated empirically through an extensive set of synthetic environments and real datasets. 
The real applications include reliable workers selection in online crowdsourcing systems, where workers might misrepresent a result for a better chance to be selected in the future; online information maximization, where nodes modify the spread to include itself into the seed set; and online recommendation systems, which characterizes the ``click framing'' behavior. Through a wide range of tasks and parameters, the experiment results corroborate our theoretical findings and demonstrate the effectiveness of our algorithms.

\subsection{Motivating Examples}
The setting of strategic manipulations describes the strategic behavior found in a variety of real applications. 
Consider a crowdsourcing platform that provides a data labeling service for payment. The platform interacts with a group of customers and a pool of workers. The customers request tasks to the platform for labeling and the platform is then responsible for selecting workers from the worker pool to complete the tasks. This process repeats, during which the platform learns the best deployment. We maintain a mild assumption that the payments from customers are non-decreasing with the quality of labels. To maximize its profit, it is desired for the platform to select workers that provides the most reliable labels. The workflow of the platform can be described in the diagram below. 

While the platform and the customers desire quality labels, it may not entirely be in the worker's interest to exert the highest effort and thus report quality labels each time, which factors into a range of reasons. Workers may adapt strategic behaviors to maximize their own utility instead. Thus, it becomes crucial to identify reliable workers to prevent strategic behaviors that jeopardize profits.
This naturally translates to a multi-armed bandits problem under strategic manipulations where the workers are the strategic arm and the payments act as the rewards. Under ideal assumptions it had shown the effectiveness of bandits algorithms on such problems \citep{jain2014quality, tran2014efficient,rangi2018multi}. 

\begin{figure}[H]
\includegraphics[width=7.5cm]{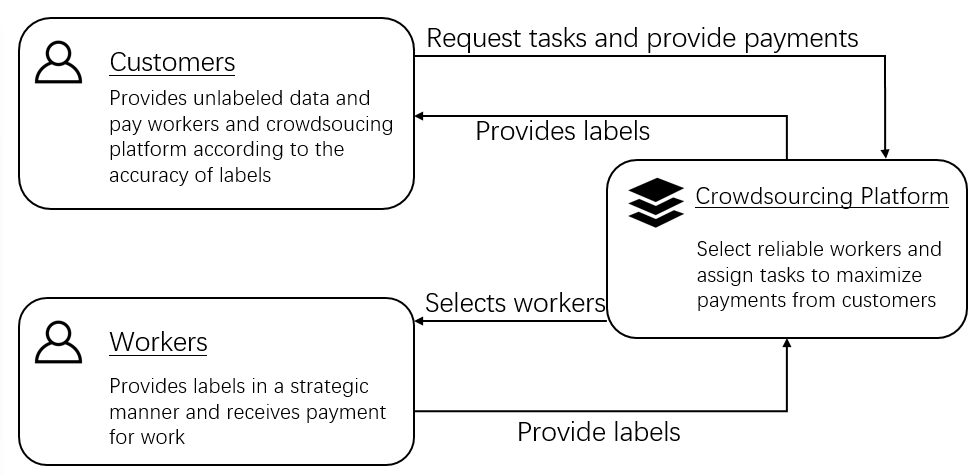}
\caption{Online crowdsourcing system}\label{crowdsys}
\end{figure}

\section{Related Work}
The problem of multi-armed bandits (MAB) was first investigated back in 1952 while some techniques utilized were developed back in 1933 \citep{10.2307/2332286, Robbins:1952, berry,auer2002finite,lattimore_szepesvari_2020}. Since then it has been extensively explored and serves as the foundation of many modern areas, including reinforcement learning, recommendation systems, graph algorithms, etc. \citep{DBLP:conf/iconip/BouneffoufBG12,pmlr-v70-vaswani17a, pmlr-v85-durand18a,10.5555/3367243.3367445,li2020stochastic}. With the need to model the selection of multiple arms in one round, MAB is then extended to combinatorial multi-armed bandits (CMAB), which see many deployments in real applications like news/goods recommendation, medical trials, routing, and so forth \citep{Wang_Ouyang_Wang_Chen_Asamov_Chang_2017,8022964,pmlr-v84-sankararaman18a, li2020online}.

Real applications motivate the community to derive algorithms in more adaptive environments. The first of which considers adversarial bandits with the classic EXP3 algorithm. The arms are assumed to be non-stationary but non-adaptive (which means that algorithms will adapt to the adversarial) \cite{auer2002nonstochastic}. Despite that adversarial bandits do not fall into the scale of this paper's related work, it leads tremendous effort to the following topics in adaptive arms.

\paragraph{Adversarial corruptions}
The adversary is given the ability to replace the observations of the bandit algorithm with arbitrary values within some constraints.
\citet{pmlr-v99-gupta19a} discuss the constraint that rewards are modified for at most $T_0$ rounds out of the total $T$ rounds. The asymptotic regret bound $O(mT_0)$ achieved in the work is shown to match the negative result, where $m$ is the number of arms. \citet{lykouris_stochastic_2018} discussed a different constraint where the corruption $z_t$ at time $t$ is cumulative up to some constant $B$.
The regret achieved under this constrained corruption is bounded by $O(mB\log(\frac{mT}{\delta}))$ with probability at least $1 - \delta$. \citet{JMLR:v20:18-395} restricts corruptions as independent Bernoulli events with probability $b$, while if corruption happens the reward becomes arbitrary and adversarial. This problem is addressed with median statistics with gap-dependent matching regret upper and lower bounds of $O(\sum_{i \neq i^{\ast}} \frac{1}{\Delta_i}\log(\frac{K}{\Delta_i}))$, where $i^{\ast}$ denotes the optimal arm and $\Delta_i$ is the suboptimality of arm $i$.

Several lines of research also discuss corruption topics in frequentist inference and partial monitoring \citep{pmlr-v83-gajane18a}, best arm identification \citep{pmlr-v99-gupta19a}, and adversarial bandits \citep{NEURIPS2020_e655c771}.

\paragraph{Adversarial attacks}
Different from arms in adversarial corruptions who intend to maximize the regret of the bandit algorithm, arms in the adversarial attack setting have the goal to maximize their number of pulls. The setting is first considered by \citet{NEURIPS2018_85f007f8}, where it shows that the attacker may spend a budget of $O(\log T)$ to deceit UCB and $\epsilon$-greedy algorithms to pull the target arms for at least $T - o(T)$ times. Stochastic and contextual MAB algorithms also suffer from undesired loss from this adversarial attack, under a variety of attack strategies \citep{pmlr-v97-liu19e}. Linear MAB algorithms, however, are shown to be near-optimal up to some linear or quadratic terms on the attack budget $B$ \citep{garcelon_adversarial_2020,bogunovic_stochastic_2020}. 

\paragraph{Strategic manipulations}
The setting of strategic manipulations further weakens the capability of the adaptive arms and prevents the bandit algorithms from being overcautious. 
The objective of an arm is still utility maximization under strategic manipulations, while each arm works on its utility individually. 
Instead of having a global coordinator for the adversarial attack strategy, the strategic arms seek the best response via the dominant Nash equilibrium. 
The strategic manipulation setting is first studied by \citet{pmlr-v99-braverman19b} where each arm is assumed to keep a portion of the reward in each round and arms maximize the cumulative reward it keeps. 
The performance of bandit algorithms will then be catastrophic under the dominant Nash equilibrium, receiving an expected cumulative reward of $0$ almost surely. 
With the utility of maximizing the number of pulls for each strategic arm, common bandits algorithms are proved to be robust with guaranteed $O(\log T)$ regret bound, but only under constant strategic budgets \citep{pmlr-v119-feng20c}. 
For $\omega(\log T)$ budgets, bandits under strategic manipulations remain an open problem and will be discussed in this paper. 

\section{Problem Formulation}
We consider the problem of combinatorial multi-armed bandits (CMAB) under the setting of strategic arms. In strategic CMAB, each arm is given a budget and the arm can strategically increase its emitted reward signals cumulatively up to this budget \emph{for its own interest}. This problem is a Stackelberg game that involves two parties. The bandit learning algorithm is deployed first to maximize its cumulative reward under the best response of the followers. The $m$ strategic bandits arms then deplete their budget where each of them aims to maximize its expected number of pulls. Knowing the principal's strategy and other followers' budget information, the followers are assumed to place their strategies according to the dominant Nash equilibrium between them. 

Formally, given a time horizon $T$, the the bandit learning algorithm is asked to pull a subset of arms, denoted by an arm subset $S_t \in \mathcal{S} $ at each time $t\in [T]$, where $\mathcal{S} = { \{0,1\}^{m}} $ is the set of all possible arm subsets. At the time $t$, based on the principal's choice of arm subset, stochastic rewards $x_{i,t}$ are generated by arm $i$ from the underlying $1$-sub-Gaussian distribution with mean $\mu_i$, independent of other arms in $S_t$. The principal does not observe these rewards. Instead, each arm can increase the emitted reward signal by an arbitrary amount $z_{i,t}\ge 0$, as long as the cumulative manipulation over the horizon does not exceed a given budget $B_i$. The principal then receives the semi-bandit feedback $\{x_{i,t}+z_{i,t}\}_{i\in S_t}$.

Let $\bm{\mu} = (\mu_1, \mu_2, ..., \mu_i) $ be the vector of expectations of all arms. The expected reward of playing any arm subset S in
any round $r_{\bm{\mu}}(S) = \mathbb{E}[R_t(S)]$,which is a function of arm subset $S$ and $\bm{\mu}$. The reward function $r_{\bm{\mu}}(S)$ is assumed to satisfy two axioms:
\begin{itemize}
    \item \textbf{Monotonicity.} The expected reward of playing any arm subset $S \in \mathcal{S}$ is monotonically non-decreasing with respect to the expected reward vector. That is, if for all $i \in [m], \mu_i \leq \mu_{i}'$, then $r_{\bm{\mu}}(S) \leq r_{\bm{\mu}'}(S)$.
    \item \textbf{Bounded smoothness.} There exists a strictly increasing function $f(\cdot)$, termed the bounded smoothness function, such that for any two expected reward vectors $\bm{\mu}$ and $\bm{\mu'}$ satisfying $\|\bm{\mu}-\bm{\mu'}\|_{\infty}\le \Lambda $, we have $|r_{\bm{\mu}} - r_{\bm{\mu'}}| \leq f({\Lambda})$.
\end{itemize}
These axioms cover a rich set of reward functions and the explicit forms of $R_t(S)$ and $r_{\bm{\mu}}(S)$ are not needed to be specified \citep{pmlr-v28-chen13a}.

Without loss of generality, assume that $S^\ast = \argmax_{S\subseteq \mathcal{S}} r_{\bm{\mu}}(S)$ is the unique optimal subset of arms. When placing reward manipulations, each strategic arm has access to its own history $h_{i,t} = \{I_{i,t'}, x_{i,t'}, z_{i,t'}\}_{t'\ge 1}$, where $I_{i,t'}$ is the indicator of whether arm $i$ is pulled at time $t'$ and $t' < t$. The strategy of arm $i$ is determined by a function that maps this history to a manipulation $z_{i,t}$, as $Z_{i,t} \colon h_{i,t-1} \to \mathbb{R}$. Without loss of generality, we assume that arms in the optimal arm subset $i\in S^\ast$ have strategic budgets of $0$, which restricts their $z_{i,t}$ to be $0$.

In the combinatorial setting, even with the exact reward vector $\bm\mu$ provided, it can be hard to exactly compute the optimized $r_{\bm{\mu}}(S)$. In view of this, many have studied probabilistic approximation algorithms in combinatorial problems, which indicates that an ($\alpha,\beta$)-approximation oracle defined below can be usually available. 

\begin{definition}[Approximation oracle]
Let $0\le \alpha,\beta \leq 1$ and define ${OPT}_{\bm{\mu}} = \max_{S\in \mathcal{S}} r_{\bm{\mu}}(S)$. 
An oracle is called an ($\alpha,\beta$)-approximation oracle if it takes an expected reward vector $\bm{\mu}$ as input and outputs an arm subset $S \in \mathcal{S}$ such that $\mathbb{P}(r_{\bm{\mu}}(S) \geq \alpha \cdot {OPT}_{\bm{\mu}}) \geq \beta$. That is, the oracle gives an arm subset $S$ that is at least as good as $\alpha$ times the reward of an optimal arm subset with probability at least $\beta$. 
\end{definition}
Denote ${S_B} = \{ r_{\bm{\mu}}(S) < \alpha \cdot {OPT}_{\bm{\mu}} | S \in \mathcal{S}\}$ to be the set of suboptimal arm subsets under the approximation oracle. 
Note that a suboptimal arm subset can be given by the oracle for two reasons. The ($\alpha,\beta$)-approximation oracle can fail, which happens with probability at most $1-\beta$. The estimation of $\bm{\mu}$ can deviate from the true value by a significant amount, resulting in accurate input to the oracle.

The objective of the principal is to maximize the expected cumulative reward before manipulation over the time horizon $T$. Equivalently, the principal minimizes the regret, the cumulative difference between the scaled optimal reward and expected actual reward, as defined below. 

\begin{definition}[Regret]
With access to an ($\alpha,\beta$)-\textit{approximation oracle}, the regret of a combinatorial bandit algorithm for $T$ rounds is
\[
Regret_{\bm{\mu},\alpha,\beta}(T) = T \cdot \alpha \cdot \beta \cdot {OPT}_{\bm{\mu}} - \mathbb{E}\big[\sum^T_{t=1}r_{\bm{\mu}}(S_t)\big]\,,
\]
where the randomness in $\mathbb{E}[\sum^T_{t=1}r_{\bm{\mu}}(S_t)]$ involves the stochasticity of the bandit algorithm and the oracle.
\end{definition}

The objective of each strategic arm $i\in [m]\setminus S^{\ast}$, however, is to maximize the number $\sum_{t=1}^T I_{i,t}$ of times it is pulled over the time horizon. To achieve this, the arm needs to confuse the principal by deviating the emitted reward signals up to the possessed budget. 


\section{Strategic Combinatorial UCB}
We now propose a variant of combinatorial upper confidence bound algorithm that is robust to strategic manipulations of rewards in Algorithm \ref{algo}. The only mild assumption we maintain is that the learning algorithm has the knowledge of the largest budget possessed among all bandits arms, i.e, $B_{max}$. This is a relaxation of the assumptions made by \citet{pmlr-v119-feng20c} on the strategic UCB algorithm, in which the learning algorithm has access to the cumulative use of budget at every time step. We start with a detailed description of the algorithm and then analyze the theoretical upper bound of regret, which enjoys $O(m \log{T} + m B_{max})$. 

For each arm $i$, our algorithm maintains a counter $K_{i,t-1}$ as the total number of times arm $i$ has been pulled up to time $t-1$ and $\widetilde{\mu}_{i,t} = \frac{\sum_{j=1}^{t-1} (x_{i,j}+ z_{i,j})}{K_{i,t-1}}$ as the empirical mean estimation based on the observations. At each time step, the algorithm computes the UCB estimation $\Bar{\mu}_{i,t}= \widetilde{\mu}_{i,t} + \sqrt{\frac{ 3 \log{t} }{2K_{i,t-1}}} + \frac{B_{max}}{K_{i,t-1}}$ for $i \in [m]$. With $\Bar{\mu}_{i,t}$, the $(\alpha, \beta)$-approximation oracle then outputs an approximately optimal arm subset $S_t \in \mathcal{S}$. The algorithm plays the return arm subset and update the counter $K_{i,t}$ and the estimation $\widetilde{\mu}_{i,t}$ accordingly. 

\begin{algorithm}[htb]
\DontPrintSemicolon
{\bf Input:}
Horizon $T$, number $m$ of arms, maximum budget $B_{max}$\;
{\bf Output:} Arm subset $S_t$\; 
  Initialize $K_{i,0} = 0$ and $\widetilde{\mu}_{i,0} = 0$ for $i \in [m]$\; 
  \For{t = 1 $\to$ m}{
    Play an arbitrary arm subset $S_t \in \mathcal{S}$ such that $t \in S_t$\; 
    $K_{i,t} = K_{i,t-1} + 1$ for $i \in S_t$\;
    $\widetilde{\mu}_j = x_{i,t}+z_{i,t}$ for $i \in S_t$ \;}
  \For{t = m+1 $\to$ T}{
  For $i \in [m]$, compute $\Bar{\mu}_{i,t}= \widetilde{\mu}_{i,t-1} + \sqrt{\frac{ 3 \log{t} }{2K_{i,t-1}}} + \frac{B_{max}}{K_{i,t-1}} $\;
  $S_t$ = {\textbf{oracle}}$(\Bar{\mu}_{1,t},\Bar{\mu}_{2,t},.....,\Bar{\mu}_{m,t})$\;
  Play arm subset $S_t$ and update 
  $K_{i,t} = K_{i,t-1} + 1$ and $ \widetilde{\mu}_{i,t} =  \frac{\sum_{j=1}^{t} (x_{i,j} + z_{i,j})}{K_{i,t}}$ for $i \in S_t$ \;
  }
\caption{SCUCB}
\label{algo}
\end{algorithm}

We first introduce a few notations that are used in our results. Define, for any arm $i\in[m]$, the suboptimality gaps as
\begin{align*}
        &\Delta^i_{min} = \alpha \cdot {OPT}_{\bm{\mu}} - \max\{r_{\bm{\mu}}(S) \mid S \in S_B, i \in S\}\,, \notag \\
     &\Delta^i_{max} = \alpha \cdot {OPT}_{\bm{\mu}} - \min\{r_{\bm{\mu}}(S) \mid S \in S_B, i \in S\}\,.
\end{align*}
We then denote the maximum and minimum of the suboptimality gaps as $\Delta_{max} = \max_{i \in [m]} \Delta^i_{max}\notag $ and $\Delta_{min} = \min_{i \in [m]} \Delta^i_{min}$.

The following lemma re-establish the canonical tail bound inequality in UCB under the setting of strategic manipulations.
\begin{lemma}\label{lambda}
Let ${\Lambda}_{i,t} = \sqrt{\frac{3 \log{t} }{2K_{i,t-1}}} + \frac{\rho_{i,t-1}}{K_{i,t-1}}$, where $\rho_{i,t-1}$ is the total strategic budget spent by arm $i$ up to time $t-1$ and $K_{i,t-1}$ be the total number of pulls of arm $i$ up to time $t-1$. Define the event $E_t = \{ |\ \widetilde{\mu}_{i,t-1} - \mu_{i}| \leq \Lambda_{i,t}, \forall i \in [m]\}$, where $\mu_i$ is the true mean of arm $i$'s underlying distribution. 
Then, $\mathbb{P}(\lnot E_t)\leq 2m \cdot t^{-2}$.
\end{lemma}

Armed with Lemma \ref{lambda}, we present one of our main theorems, Theorem \ref{appendixA}, which gives the regret bound of $O(\log T)$ of SCUCB. The outline of the proof follows that of CUCB by \citet{pmlr-v28-chen13a}. To complete the proof, we carefully choose $\psi_t$, which controls the trade-off between the exploration and exploitation periods.
\begin{theorem}
\label{appendixA}
The regret of the SCUCB algorithm with $m$ strategic arms in time horizon $T$ using an $(\alpha,\beta)$-approximation oracle is at most 
\begin{align}
 m\cdot\Delta_{max}\left( \frac{8B_{max} f^{-1}(\Delta_{min}) + 6\log{T}} {\left(f^{-1}(\Delta_{min})\right)^2} + \frac{\pi^2}{3}+1 \right)\,, \notag \nonumber
\end{align}
where $f^{-1}(\cdot)$ is the inverse bounded smoothness function.
\end{theorem}

\begin{proof}[Proof sketch]
We introduce a counter $N_i$ for each arm $i\in[m]$ after the $m$-round initialization and let $N_{i,t}$ be the value of $N_i$ at time $t$. We initialize $N_{i,m} = 1$. By definition, $\sum_{i \in [m]} N_{i,m} = m$. For $t > m$, the counter $N_{i,t}$ is updated as follows:
\begin{itemize}
    \item If $S_t \in S_B$, then $N_{{i'},t} = N_{{i'},t-1} + 1$ where $i' =$ $\argmin_{i \in S_t}$ $N_{i,t-1}$. In the case that $i'$ is not unique, we break ties arbitrarily;
    \item If $S_t \notin S_B$, then no counters will be updated. 
\end{itemize}  
As such, the total number of pulls of suboptimal arm subsets is less than or equal to $\sum_{i=1}^m N_{i,T}$.

Define $\psi_t = \frac{8B_{max}f^{-1}(\Delta_{min}) + 6\log{t}} {\left(f^{-1}(\Delta_{min})\right)^2} > c $, where $c$ is the larger solution of
\begin{align*}
\label{square}
    &(f^{-1}(\Delta_{min}))^2 c^2  + 16B_{max}^2 - ((8B_{max} f^{-1}(\Delta_{min})- 6\log{t}) c = 0 \,. 
\end{align*}
Then, we decompose the total number $\sum_{i=1}^m N_{i,T}$ of pulls of suboptimal arm subsets as
\begin{align}
\sum_{i=1}^m N_{i,T} = \ & m +\sum_{t=m+1}^T \mathbb{I}\{ S_t \in S_B\} \notag \\
=\ &m +\sum_{t=m+1}^T \sum_{i=1}^m \mathbb{I}\{ S_t \in S_B, N_{i,t} > N_{i,t-1}, N_{i,t-1} \leq \psi_t\} \notag \\
&+ \sum_{t=m+1}^T \sum_{i=1}^m \mathbb{I}\{ S_t \in S_B, N_{i,t} > N_{i,t-1}, N_{i,t-1} > \psi_t\}\notag \\
\leq \ & m + m\psi_T+\sum_{t=m+1}^T \mathbb{I}\{ S_t \in S_B,N_{i,t-1} > \psi_t, \forall i \in S_t\}\,. \notag
\end{align}

The inequality follows as $\sum_{t=m+1}^T \sum_{i=1}^m \mathbb{I}\{ S_t \in S_B, N_{i,t} > N_{i,t-1}, N_{i,t-1} \leq \psi_t\}$ can be trivially bounded by $m\psi_T$. Thus the key to bound the total number of pulls of suboptimal arm subset is to upper bound $\sum_{t=m+1}^T \mathbb{I}\{ S_t \in S_B,N_{i,t-1} > \psi_t, \forall i \in S_t\}$. Let $F_t$ denotes the event that the oracle fails to provide an $\alpha$-approximate arm subset with respect to the input vector $\Bar{\bm{\mu}} = (\Bar{\mu}_{1,t},\Bar{\mu}_{2,t},.....,\Bar{\mu}_{m,t})$. 
Then we can decompose $\sum_{t=m+1}^T \mathbb{I}\{ S_t \in S_B,N_{i,t-1} > \psi_t, \forall i \in S_t\}$ as 
\begin{align}
&\quad \sum_{t=m+1}^T \mathbb{I}\{ S_t \in S_B,N_{i,t-1} > \psi_t, \forall i \in S_t\}\notag\\
&\leq \sum_{t=m+1}^T (\mathbb{I}\{F_t\} + \mathbb{I}\{ \lnot F_t, S_t \in S_B,  K_{i,t-1} > \psi_t, \forall i \in S_t\})\notag \\
&\leq (T-m)(1-\beta) +  \mathbb{I}\{ \lnot F_t, S_t \in S_B, K_{i,t-1} > \psi_t, \forall i \in S_t\}\notag \,. \notag
\end{align}

By leveraging the monotonicity and smoothness assumptions of the reward function, we show that $\mathbb{P}(\{\lnot F_t, E_t,  S_t \in S_B, K_{i,t-1} > \psi_t, \forall i \in S_t\}) = 0$. Therefore, by the inclusion-exclusion principle,
\begin{align}
     \mathbb{P}(\{\lnot F_t, S_t \in S_B, K_{i,t-1} > \psi_t, \forall i \in S_t\}) 
     &\leq \mathbb{P}(\lnot E_t) \leq 2m \cdot t^{-2} \,.\notag
\end{align}

Leveraging the upper bound of total number of pulls of suboptimal arm subsets, the regret is bounded by
\begin{align}
    & \quad Regret_{\bm{\mu},\alpha,\beta}(T)
    \notag \\
    & \leq T\alpha \beta \cdot \text{OPT}_{\bm{\mu}} - \left( T\alpha\cdot \text{OPT}_{\bm{\mu}} - \mathbb{E} \left[ \sum_{i=1}^m N_{i,T} \right]\cdot \Delta_{max} \right)
    \notag \\
     &\leq m\cdot \Delta_{max}\left(\frac{8B_{max} f^{-1}(\Delta_{min}) + 6\log{T}} {\left(f^{-1}(\Delta_{min})\right)^2} + \frac{\pi^2}{3}+1\right)\,. \notag \qedhere
\end{align}
\end{proof}

It remains in question whether a bandit algorithm can achieve a regret upper bound sublinear in $B_{max}$. 
Our conjecture is negative.
In fact, under strategic manipulations of rewards, the design of robust bandit algorithms, e.g. UCB and $\epsilon$-greedy, is analogous to the design of outlier-robust mean estimation algorithms. Existing works on robust mean estimation, such as \citet{steinhardt_et_al:LIPIcs:2018:8368}, argue that from an information theoretical point of view mean estimation error must be depending on the variance of the data. Casting this argument to bandit with the strategic manipulation setting, we believe that a tight regret bound is unlikely to be independent of the strategic budget $B_{max}$. Moreover, Proposition 4 in \citet{steinhardt_et_al:LIPIcs:2018:8368} implies that the dependency on $B_{max}$ is linear. This corresponds to our linear dependency of regret in Theorem \ref{appendixA}.

\section{Lower Bounds for strategic budget}
To fully understand the effects of strategic manipulations, we investigate the relationship between the strategic budget and the performance of UCB-based algorithms. Our results provide the dependency between an arm's strategic budget and the number of times it is pulled, which influence the regret of the algorithm incurred by the arm. Our analysis gains some insight from \citet{zuo_near_2020}, which limits the discussion to the 2-armed bandit setting and cannot be applied to the general MAB setting directly. 

We first define some notations used in the theorem.
Without loss of generality, let arm $i^{\ast}$ be the optimal arm and arm $i, i \neq i^{\ast}$ be an arbitrary strategic suboptimal arm. Assume that arm $i$ has no access to the information regarding other arms. Let $K_{i,t}$ and $K_{i^{\ast},t}$ denote the number of times arm $i$ and arm $i^{\ast}$ have been pulled up to time $t$, respectively. Denote $\hat{\mu}_{i,t}$ as the empirical estimate of the underlying mean $\mu$ without manipulations, i.e., $\hat{\mu}_{i,t} = \frac{\sum^{t-1}_{j=1}x_{i,j}}{K_{i,t-1}}$. Define the suboptimality gap for each strategic arm $i$ to be $\delta_i = \mu_{i^{\ast}} - \mu_i$. We use a slightly revised UCB algorithm as the basic algorithm, where the UCB estimation term for arm $i$ is $\hat{\mu}_{i,t} + \sqrt{\frac{2\log(K_{i,t}^2/\eta^2)}{K_{i,t}}}$ and $\eta$ is a confidence parameter chosen by the algorithm.

\begin{theorem}\label{lower}
In stochastic multi-armed bandit problems, for a strategic suboptimal arm $i$ without access to other arms' information, to be pulled for $\omega(k)$ in $T$ steps under the UCB algorithm where $k\ge O(\log{T})$, the minimum strategic budget is $\omega(k)$. 
\end{theorem}

This dependency of $k$ can be extended to CMAB and CUCB straightforwardly when arms within an arm subset collude. Counter-intuitively, the dependency between the number of pulls of a strategic arm and its strategic budget is linear and subsequently, this infers that for a strategic arm to manipulate the algorithm to suffer an undesired regret of $\omega(k)$, where $k\ge O(\log T)$, the strategic budget must be at least $\omega(k)$.

\section{Experiments}
In this section, we evaluate our SCUCB algorithm empirically on synthetic data and three real applications, namely online worker selection in crowdsourcing, online recommendation, and online influence maximization. We highlight the best performance among all algorithms with bold text. 

\subsection{Baseline Algorithms}

We compare our proposed SCUCB algorithm with both stochastic and adversarial bandits algorithms that achieves optimal asymptotic regret in their settings.
\begin{enumerate}
\allowdisplaybreaks
    \item CUCB \cite{pmlr-v28-chen13a}. CUCB is the naive counterpart of our algorithm. The algorithm calculates an upper confidence interval for each arm and picks the best arm subset with the highest upper confidence interval. 
    \item TSCB \cite{wang2018thompson}. TSCB is the combinatorial version of the classical Thompson sampling algorithm. The algorithm maintains a prior beta distribution estimation for each arm and updates according to the received reward. At each time, the algorithm samples from the estimated distributions and pick actions according to the highest sample. 
    \item Combinatorial variant of EXP3 \cite{auer_nonstochastic_2002}. The algorithm maintains a weight for each arm and draws actions according to the normalized weight distribution. Upon receiving rewards, the algorithm update weight according to the classical EXP3 update rule. 
\end{enumerate}

\subsection{Synthetic Experiments}
We conduct experiments presented in this section with synthetic data and compare our proposed algorithm with its naive counterpart. The approximation oracle is designed to succeed with probability $1$. Each bandit arm is modeled to follow a Bernoulli distribution with randomly populated $\mu \in [0,1]$ and all arms adapt LSI strategy. 
The arms in the optimal arm subset have a strategic budget of $0$ since an additional strategic budget for optimal arms would only boost the performance of our algorithm. All other arms are equipped with a randomly allocated budget $B$, $ 0 \leq B \leq B_{max}$ by definition. 
As is typical in the bandit literature, for example in \citet{auer2002finite}, we evaluate both the naive and strategic CUCB algorithms on their tuned versions, where the UCB exploration parameter is scaled by a constant factor $\gamma \in [0,1]$. To ensure reproducible results, each experiment is repeated for $10$ random seeds and the averaged result is presented.  

The first set of experiment is conducted with $K = \{10, 20\}$ arms through a time horizon of $T=5000$ time steps with maximum possible budget of $B_{max} = \{70, 90, 110, 130\}$. The algorithms are asked to select $\text{Action size} = 2$ arms as an arm subset at each time step. The cumulative regret incurred by CUCB and SCUCB algorithm is presented in the table \ref{table:1} where the best performance is highlighted in bold text. Clearly, SCUCB demonstrated its effectiveness as it achieves significantly smaller regrets in various possible maximum strategic budgets. 
\begin{table}[H]
\centering
\begin{tabularx}{0.47\textwidth}{@{}lY@{\extracolsep{\fill}}YYY@{}}
\toprule
\multicolumn{5}{c}{Cumulative regret, Action size = 2, K = 10}                                       \\ \midrule
\multicolumn{1}{l|}{Dataset}        & \multicolumn{4}{c}{Synethetic}                                        \\ \midrule
\multicolumn{1}{l|}{Bmax}           & 70              & 90              & 110             & 130             \\
\multicolumn{1}{l|}{CUCB}           & 171.74          & 187.51          & 259.04          & 256.66          \\
\multicolumn{1}{l|}{SCUCB} & \textbf{143.85} & \textbf{172.57} & \textbf{208.53} & \textbf{233.57} \\ \bottomrule
\end{tabularx}
\begin{tabularx}{0.47\textwidth}{@{}lY@{\extracolsep{\fill}}YYY@{}}
\toprule
\multicolumn{5}{c}{Cumulative regret, Action size = 2, K = 20}                                        \\ \midrule
\multicolumn{1}{l|}{Dataset}        & \multicolumn{4}{c}{Synethetic}                                         \\ \midrule
\multicolumn{1}{l|}{Bmax}           & 70              & 90              & 110             & 130              \\
\multicolumn{1}{l|}{CUCB}           & 434.82          & 492.02          & 520.97          & 549.57           \\
\multicolumn{1}{l|}{SCUCB} & \textbf{301.57} & \textbf{365.94} & \textbf{450.10} & \textbf{505.453} \\ \bottomrule
\end{tabularx}
\caption{Cumulative regret achieved by CUCB and SCUCB with synthetic data with various action size.}\label{table:1}
\end{table}

The next two tables reveals the advantage of SCUCB algorithm with various size of action set. The experiments are conducted with time horizon of $T=5000$ time steps with $K = \{10, 20\}$ arms and maximum possible budget of $B_{max} = \{50\}$. The algorithms are asked to select $\text{Action size} = \{2, 4, 6, 8\}$ arms as an arm subset at each time step. Once again, our SCUCB outperforms its naive counterpart by achieving much smaller cumulative regret, 
the best numerical results across algorithms are highlighted in bold text. 
\begin{table}[H]
\centering
\begin{tabularx}{0.47\textwidth}{@{}lY@{\extracolsep{\fill}}YYY@{}}
\toprule
\multicolumn{5}{c}{Cumulative regret, Bmax = 50, K = 10}                                            \\ \midrule
\multicolumn{1}{l|}{Dataset}        & \multicolumn{4}{c}{Synethetic}                                       \\ \midrule
\multicolumn{1}{l|}{Action Size}    & 2               & 4               & 6               & 8              \\
\multicolumn{1}{l|}{CUCB}           & 140.21          & 150.88          & 158.22          & 123.91         \\
\multicolumn{1}{l|}{SCUCB} & \textbf{105.58} & \textbf{113.63} & \textbf{103.42} & \textbf{88.72} \\ \bottomrule
\end{tabularx}
\begin{tabularx}{0.47\textwidth}{@{}lY@{\extracolsep{\fill}}YYY@{}}
\toprule
\multicolumn{5}{c}{Cumulative regret, Bmax = 50, K = 20}                                             \\ \midrule
\multicolumn{1}{l|}{Dataset}        & \multicolumn{4}{c}{Synethetic}                                        \\ \midrule
\multicolumn{1}{l|}{Action Size}    & 2               & 4               & 6               & 8               \\
\multicolumn{1}{l|}{CUCB}           & 302.10          & 316.09          & 340.52          & 328.80          \\
\multicolumn{1}{l|}{SCUCB} & \textbf{232.23} & \textbf{247.32} & \textbf{268.27} & \textbf{306.77} \\ \bottomrule
\end{tabularx}
\caption{Cumulative regret achieved by UCB and SCUCB with synthetic data with various $B_{max}$.}
\end{table}
To compare the two algorithms in detail, we plot the cumulative regret incurred by algorithms against the time steps. It becomes apparent that the SCUCB algorithm features a cumulative regret versus time steps line that is much smaller than the one caused by CUCB even under a higher maximum possible strategic budget. We compare the two algorithm under settings where the maximize possible budget is $B_{max} = 70, 90$ and $B_{max} = 110, 130$.
\begin{figure}[H]
\centering
\subfigure[]{
\label{Fig1}
\includegraphics[width=3.8cm]{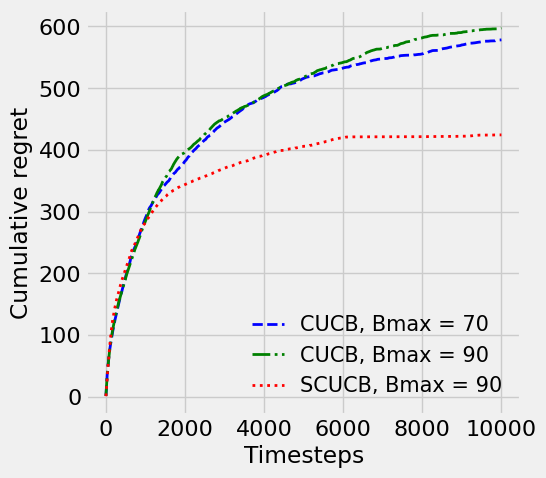}}
\subfigure[]{
\label{Fig2}
\includegraphics[width=3.8cm]{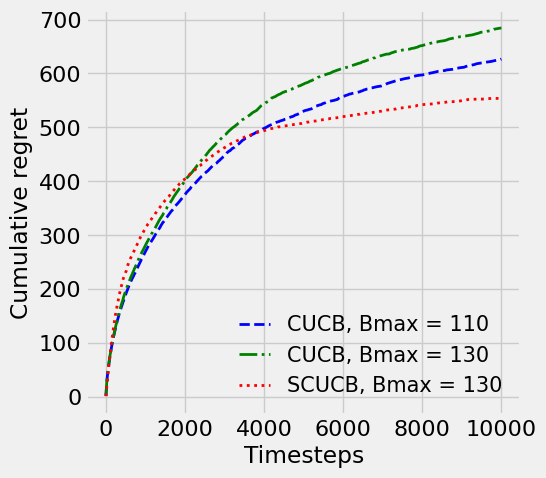}}
\caption{\ref{Fig1} Comparison of CUCB and SCUCB with $B_{max} = 70, 90$. \ref{Fig2} Comparison of CUCB and SCUCB with $B_{max} = 110, 130$.}
\end{figure}

\subsection{Online Worker Selection in Crowdsourcing Systems}
We simulated an online crowdsourcing system that has a workflow resembled by figure \ref{crowdsys} and model it as a combinatorial bandits problem. We performed an extensive empirical analysis of our SCUCB algorithm against CUCB, the combinatorial version of Thompson sampling, and EXP3. The experiments were conducted with the Amazon sentiment dataset, where Amazon item reviews are labeled 'is/is not book' or 'is/is not negative \cite{Krivosheev:2018:CCM:3290265.3274366}. We split the dataset into two, where 'is book' contains data that only has labels 'is/is not book' and 'is negative' contains data that only has labels 'is/is not negative'. Both datasets consists of 7803 reviews, 284 workers, and 1011 tasks. 
The correctness of workers and tasks can be visualized by the figures as below. Notice that most of the workers attain an accuracy of $60$ percent and above. This highlights that most of the workers we are interacting with has the ability to label most of the tasks correctly.
\begin{figure}[H]
\centering
\subfigure[]{
\label{worker}
\includegraphics[width=3.8cm]{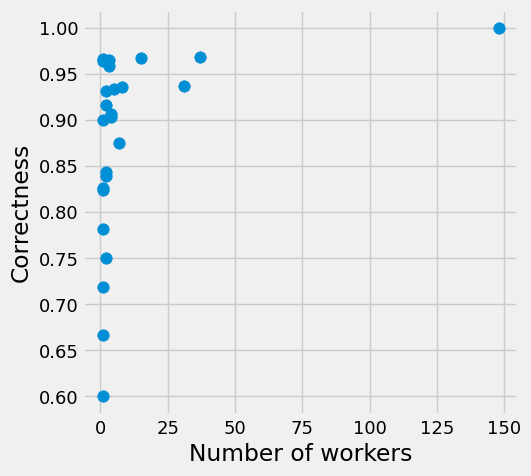}}
\subfigure[]{
\label{task}
\includegraphics[scale = 0.28]{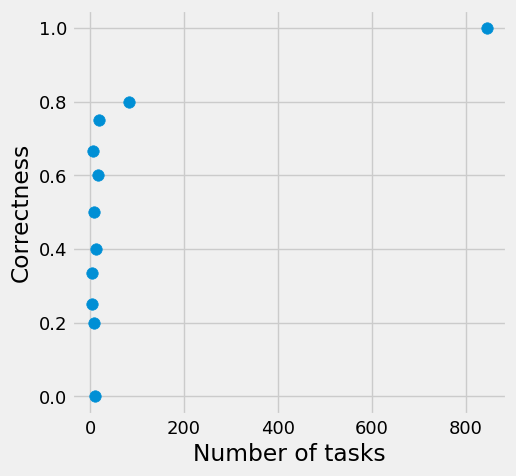}}
\caption{\ref{worker} Visualization of correctness of worker's labels. \ref{task}Visualization of correctness of task's labels.}
\end{figure}

For each task, the data contains responses from 5 workers and the crowdsourcing platform is asked to select 2, 3, or 4 workers at each time step. To measure the performance of the algorithms, we made the assumption that the reward is monotonically correlated to the accuracy of the label compared to the ground truth label. Thus, we define the reward to be 1, if the worker's label is the same as ground truth label and 0 otherwise. To model the strategic behavior of the workers, we model each worker with a randomly allocated strategic level $s \in [0,1]$ such that the worker provides its honest response with probability $s$. 

For the first set of experiments, we choose the maximum possible strategic budget to be $B_{max} = 500$. For actions size of $2, 3, 4$. We obtain the following cumulative rewards and cumulative regret results where the best performance is highlighted in bold text. To ensure reproducible results, each experiment is repeated 5 times and the averaged result is presented.

To further investigate the effect of the maximum possible strategic budget on the performance of algorithms, we fix the size of action size to be 2 and plot the cumulative reward/cumulative regret incurred by each algorithm again $B_{max} \in [100, 700]$. Regardless of the $B_{max}$ value, our proposed SCUCB algorithm consistently outperforms other methods. For better visualization, we omit the line for combinatorial Thompson sampling due to its relatively weaker performance. 

\begin{figure}[H]
\centering
\subfigure[]{
\label{culre}
\includegraphics[width=3.8cm]{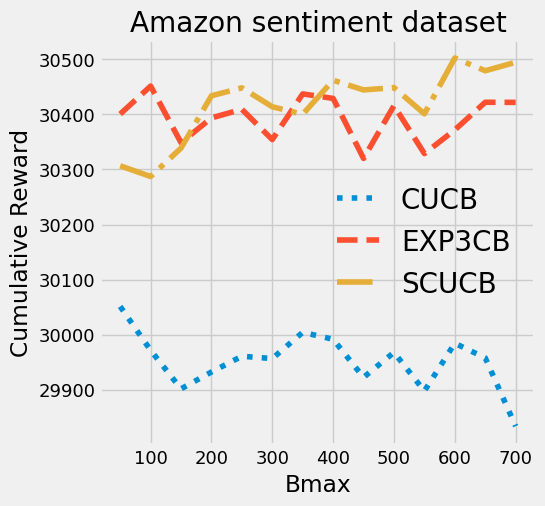}}
\subfigure[]{
\label{culreg}
\includegraphics[width=3.8cm]{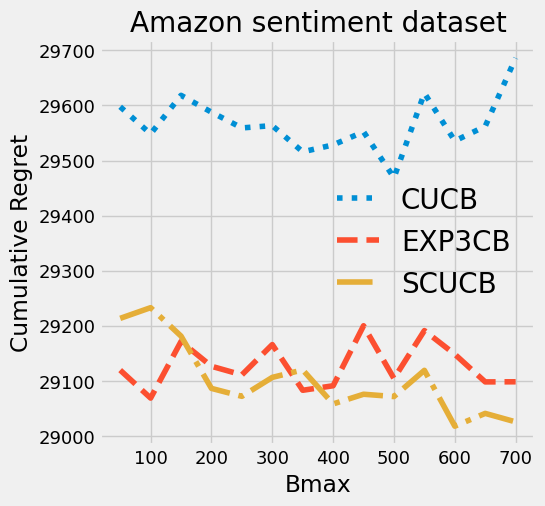}}
\caption{\ref{culre} Cumulative rewards attained by algorithms with various level of $B_{max}$. \ref{culreg} Cumulative regrets attained by algorithms with various level of $B_{max}$.  }
\end{figure}

The results of our crowdsourcing experiments reveal the robustness of our SCUCB algorithm under strategic manipulations. It also indicates the performance of adversarial and stochastic bandits algorithms under strategic manipulations. The combinatorial variant of Thompson sampling is vulnerable under manipulations, which agrees with our expectation given that the algorithm is a Bayesian algorithm in nature and is heavily relying on estimating the underlying distribution of arms. Its estimation can be easily perturbed under strategic manipulation, and thus leads to undesired performance results. It is also expected that the SCUCB algorithm far outperforms its naive counterpart, CUCB algorithm. The combinatorial version of EXP3 algorithm is the most competitive algorithm with our SCUCB. As the EXP3 algorithm was originally designed for a pure adversarial setting, the algorithm has some robustness under strategic manipulations. Shown in the Figure \ref{culre} and Figure \ref{culreg}, the combinatorial EXP3 algorithm consistently shows robustness across various levels of maximum possible strategic budget. However, as our SCUCB algorithm considers the maximum possible strategic budget and thus more adaptive, SCUCB far outperforms EXP3CB as $B_{max}$ increases. 
\begin{table}[H]
\label{first}
\centering
\begin{tabularx}{0.47\textwidth}{@{}lY@{\extracolsep{\fill}}YY@{}}
\toprule
\multicolumn{4}{c}{Crowdsourcing (Cumulative Reward), Bmax = 500}                            \\ \midrule
\multicolumn{1}{l|}{Dataset}        & \multicolumn{3}{c}{is\_book}                           \\ \midrule
\multicolumn{1}{l|}{Action size}    & 2                & 3                & 4                \\ \midrule
\multicolumn{1}{l|}{CUCB}           & 31672.6          & 46181.2          & 68244.6          \\
\multicolumn{1}{l|}{TSCB}          & 19853.0          & 34226.2          & 60621.2          \\
\multicolumn{1}{l|}{EXP3}           & 31569.6          & 46570.6          & 68482.2          \\
\multicolumn{1}{l|}{SCUCB} & \textbf{32082.6} & \textbf{46623.6} & \textbf{68524.6} \\ \bottomrule
\end{tabularx}
\begin{tabularx}{0.47\textwidth}{@{}lY@{\extracolsep{\fill}}YY@{}}
\toprule
\multicolumn{4}{c}{Crowdsourcing (Cumulative Regret), Bmax = 500}                            \\ \midrule
\multicolumn{1}{l|}{Dataset}        & \multicolumn{3}{c}{is\_book}                           \\ \midrule
\multicolumn{1}{l|}{Action size}    & 2                & 3                & 4                \\ \midrule
\multicolumn{1}{l|}{CUCB}           & 27787.4          & 42198.8          & 48125.4          \\
\multicolumn{1}{l|}{TSCB}          & 39607.0          & 54153.8          & 55748.8          \\
\multicolumn{1}{l|}{EXP3}           & 27890.4          & 41809.4          & 47887.8          \\
\multicolumn{1}{l|}{SCUCB} & \textbf{27377.4} & \textbf{41756.4} & \textbf{47845.4} \\ \bottomrule
\end{tabularx}

\begin{tabularx}{0.47\textwidth}{@{}lY@{\extracolsep{\fill}}YY@{}}
\toprule
\multicolumn{4}{c}{Crowdsourcing (Cumulative Reward), Bmax = 500}                            \\ \midrule
\multicolumn{1}{l|}{Dataset}        & \multicolumn{3}{c}{is\_negative}                       \\ \midrule
\multicolumn{1}{l|}{Action size}    & 2                & 3                & 4                \\ \midrule
\multicolumn{1}{l|}{CUCB}           & 31108.8          & 44757.6          & 59424.8          \\
\multicolumn{1}{l|}{TSCB}          & 17240.6          & 32010.4          & 51109.0          \\
\multicolumn{1}{l|}{EXP3}           & 30879.0          & 44767.2          & 58207.8          \\
\multicolumn{1}{l|}{SCUCB} & \textbf{31759.0} & \textbf{44843.4} & \textbf{59441.4} \\ \bottomrule
\end{tabularx}

\begin{tabularx}{0.47\textwidth}{@{}lY@{\extracolsep{\fill}}YY@{}}
\toprule
\multicolumn{4}{c}{Crowdsourcing (Cumulative Regret), Bmax = 500}                            \\ \midrule
\multicolumn{1}{l|}{Dataset}        & \multicolumn{3}{c}{is\_negative}                       \\ \midrule
\multicolumn{1}{l|}{Action size}    & 2                & 3                & 4                \\ \midrule
\multicolumn{1}{l|}{CUCB}           & 28411.2          & 43566.2          & 56465.2          \\
\multicolumn{1}{l|}{TSCB}          & 42279.4          & 56279.6          & 64781.0          \\
\multicolumn{1}{l|}{EXP3}           & 28641.0          & 43523.4          & 57682.2          \\
\multicolumn{1}{l|}{SCUCB} & \textbf{27761.0} & \textbf{43446.6} & \textbf{56448.6} \\ \bottomrule
\end{tabularx}

\caption{Cumulative rewards and cumulative regret attained by SCUCB and baseline algorithms with various action sizes on isbook dataset and isnegative dataset.}
\end{table}
\subsection{Online Recommendation System}

The online recommendation is a classic example of combinatorial bandits in applications. We evaluated algorithms in the latest MovieLens dataset which contains 9742 movies and 100837 ratings of the movies with each rating being $1 - 5$ \cite{10.1145/2827872}. We compare the algorithms based on the total values of ratings received, i.e. recommending a movie and receiving a rating of 5 is more desired than recommending a movie and receiving a rating of 1. 

As the dataset may consist of unbalanced data for each movie, we adapted collaborative filtering and $k$-means clustering into our evaluation system. We used the collaborative filtered for training and testing and instead of recommending movies, we clustered the movies by $k$-means clustering and asked the bandits algorithm to choose one cluster at each time. The movie with the highest rating in the chosen cluster is recommended. The workflow of the experiment setup is summarized in the above diagram. 

We performed the evaluation with various number of clusters (10, 20 or 30) and $B_{max} = 30, 50, 70$ through a time horizon of $T = 500$. The results presented below are averages over 5 runs to ensure reproducibility. The best performance across algorithms is highlighted in bold text. From the tables below, we conclude the effectiveness of SCUCB algorithm. 

\begin{table}[tbh]
\centering
\begin{tabularx}{0.47\textwidth}{@{}lY@{\extracolsep{\fill}}YY@{}}
\toprule
\multicolumn{4}{c}{Recsys (Cumulative Reward), T=500, 10 clusters}                          \\ \midrule
\multicolumn{1}{l|}{Dataset}       & \multicolumn{3}{c}{Movielens}                          \\ \midrule
\multicolumn{1}{l|}{Bmax}          & 30               & 50               & 70               \\ \midrule
\multicolumn{1}{l|}{UCB}           & 2297.60          & 2494.36          & 2631.02          \\
\multicolumn{1}{l|}{TS}            & 2160.14          & 2314.76          & 2468.26          \\
\multicolumn{1}{l|}{EXP3}          & 2181.11          & 2449.28          & 2430.96          \\
\multicolumn{1}{l|}{SUCB} & \textbf{2305.75} & \textbf{2314.76} & \textbf{2636.07} \\ \bottomrule
\end{tabularx}

\begin{tabularx}{0.47\textwidth}{@{}lY@{\extracolsep{\fill}}YY@{}}
\toprule
\multicolumn{4}{c}{Recsys (Cumulative Reward), T=500, 20 clusters}                          \\ \midrule
\multicolumn{1}{l|}{Dataset}       & \multicolumn{3}{c}{Movielens}                          \\ \midrule
\multicolumn{1}{l|}{Bmax}          & 30               & 50               & 70               \\ \midrule
\multicolumn{1}{l|}{UCB}           & 2469.93          & 2336.92          & 2159.25          \\
\multicolumn{1}{l|}{TS}            & 2260.20          & 1850.55          & 1905.42          \\
\multicolumn{1}{l|}{EXP3}          & 2380.74          & 2317.82          & 2025.39          \\
\multicolumn{1}{l|}{SUCB} & \textbf{2474.69} & \textbf{2341.19} & \textbf{2177.70} \\ \bottomrule
\end{tabularx}
\begin{tabularx}{0.47\textwidth}{@{}lY@{\extracolsep{\fill}}YY@{}}
\toprule
\multicolumn{4}{c}{Recsys (Cumulative Reward), T=500, 30 clusters}                                       \\ \midrule
\multicolumn{1}{l|}{Dataset}       & \multicolumn{3}{c}{Movielens}                          \\ \midrule
\multicolumn{1}{l|}{Bmax}          & 30               & 50               & 70               \\\midrule
\multicolumn{1}{l|}{UCB}           & 2443.57          & 2393.88          & 2472.31          \\
\multicolumn{1}{l|}{TS}            & 2132.16          & 2248.56          & 1855.14          \\
\multicolumn{1}{l|}{EXP3}          & 2368.42          & 2265.32          & 2339.49          \\
\multicolumn{1}{l|}{SUCB} & \textbf{2436.43} & \textbf{2397.72} & \textbf{2476.52} \\ \bottomrule
\end{tabularx}
\caption{Cumulative rewards attained by algorithms with different number of clusters $10, 20, 30$ and different level of $\{B_{max} = 30, 50, 70\}$ across 500 time steps. }
\end{table}

\begin{figure}[tbh]
\centering
\includegraphics[width=7cm]{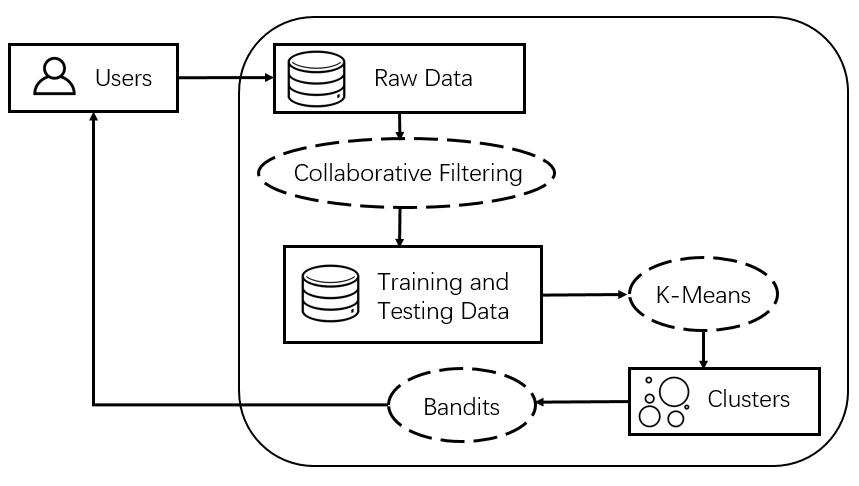}
\caption{Experiment set up of online recommendation.}
\end{figure}
\subsection{Online Influence Maximization}

We implemented the online influence maximization with an offline influence maximization algorithm TIM as our oracle \cite{10.1145/2588555.2593670}. TIM is one of the offline influence maximization algorithms that achieve asymptotic optimality. We perform the experiments on two datasets. One is a simulation dataset with 16 nodes and 44 edges, the other is a 100 nodes subset from Digg dataset, where each node represents a user from Digg website \cite{nr}. We simulate the connectivity of edges at each time step by randomly assigning each edge a connectivity probability at the start of the experiment. 
\begin{figure}[hbt!]
\centering
\includegraphics[width=0.3\textwidth]{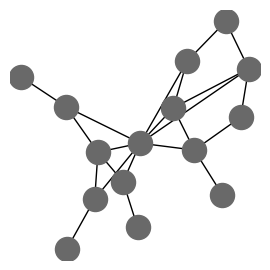}
\caption{Visualization of synthetic graph for evaluation.}
\label{test_graph}
\end{figure}
\begin{figure}[hbt!]
\centering
\includegraphics[width=0.35\textwidth]{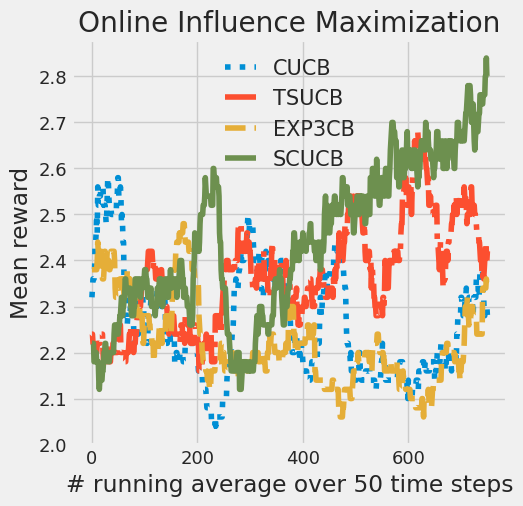}
\caption{ Mean rewards attained by algorithms.}\label{16c2}
\end{figure}

Figure \ref{test_graph} visualizes the synthetic graph we created for evaluation. We then tested the algorithms over $T = 800$ time steps with action size of 2 and maximum possible strategic budget of $B_{max} = 300$. Figure \ref{16c2} shows the effectiveness of algorithms by evaluating based on the averaged nodes influenced. To ensure reproducible results, we calculate the mean of influence spread over 100 trials. For better visualization, we kept a running mean of over 50 time steps to smooth the lines. 

We then investigate the effect of action size and the results are summarized in the following tables \ref{table:5}.
\begin{table}[H]
\centering
\begin{tabularx}{0.47\textwidth}{@{}lY@{\extracolsep{\fill}}YYYY@{}}
\toprule
\multicolumn{5}{c}{OIM (Averaged Final Reward), Bmax = 200}                                   \\ \midrule
\multicolumn{1}{l|}{Dataset}        & \multicolumn{3}{c}{Synthetic, 16 nodes, 44 edges}         \\ \midrule
\multicolumn{1}{l|}{Action size}       & 2      & 4      & 6       & 8       \\ \midrule
\multicolumn{1}{c|}{CUCB}   & 5.46          & 6.54          & 10.4           & 12.5           \\
\multicolumn{1}{c|}{TSCB}  & 6.7           & 8.50           & 10.44          & 12.7           \\
\multicolumn{1}{c|}{EXP3CB} & 5.86          & 8.14          & 10.34          & 12.46          \\
\multicolumn{1}{c|}{SCUCB}  & \textbf{7.44} & \textbf{9.32} & \textbf{11.16} & \textbf{13.04} \\ \bottomrule
\end{tabularx}

\end{table}
\begin{table}[H]\centering

\begin{tabularx}{0.47\textwidth}{@{}lY@{\extracolsep{\fill}}YY@{}}
\toprule
\multicolumn{4}{c}{OIM (Averaged Final Reward), Bmax = 200}                       \\ \midrule
\multicolumn{1}{l|}{Dataset}        & \multicolumn{3}{c}{digg, 100 nodes}         \\ \midrule
\multicolumn{1}{l|}{Action size}    & 10             & 15             & 20        \\ \midrule
\multicolumn{1}{l|}{CUCB}           & 15.64          & 20.64          & 25.32           \\
\multicolumn{1}{l|}{TSCB}          & 16.30          & 21.62          & 26.08          \\
\multicolumn{1}{l|}{EXP3}           & 12.16          & 17.76          & 21.72          \\
\multicolumn{1}{l|}{SCUCB} & \textbf{16.56} & \textbf{22.16} & \textbf{26.14} \\ \bottomrule
\end{tabularx}
\caption{Averaged final rewards attained by algorithms with synthetic and Digg dataset with action size of 2, 800 time steps and $B_{max} = 200$.}\label{table:5}
\end{table}

\section{Conclusion}
We investigate the problem of combinatorial MAB under strategic manipulations of rewards. We propose a variant of the UCB algorithm, SCUCB, which attains a regret at most $O(m \log T + m B_{max})$. 
Compared to previous studies on the bandit problems under strategic manipulations, we relax on the assumption that the algorithm has access to the cumulative strategic budget spent by each arm at each time step. 
For the robustness of bandit algorithms, we present lower bounds on the strategic budget for a malicious arm to incur a $\omega(\log T)$ regret of the bandit algorithm. 

We provide extensive empirical results on both synthetic and real datasets with a range of applications to verify the effectiveness of the proposed algorithm. Our algorithm consistently outperforms baseline algorithms that were designed for stochastic and adversarial settings.


\bibliography{sample-base}

\appendix
\providecommand{\upGamma}{\Gamma}
\providecommand{\uppi}{\pi}
\onecolumn
\section{Regret analysis of strategic UCB}

\begin{namedtheorem}[Lemma \ref{lambda} (Re-statement)]
Let $\Lambda_{i,t} = \sqrt{\frac{3 \log{t} }{2K_{i,t-1}}} + \frac{\rho_{i,t-1}}{K_{i,t-1}}$
, where $\rho_{i,t-1}$ is the total strategic budget spent by arm $i$ up to time $t-1$ and $K_{i,t-1}$ be the total number of pulls of arm $i$ up to time $t-1$. Define the event $E_t = \{  |\ \widetilde{\mu}_{i,t-1} - \mu_{i}| \leq \Lambda_{i,t}, \forall i \in [m]\}$, where $\mu_i$ is the true mean of arm $i$'s underlying distribution. 
Then, $\mathbb{P}(\lnot E_t)\leq 2m \cdot t^{-2}$.
\end{namedtheorem}

\begin{proof}
Denote $\hat{\mu}_{i,t-1} = \frac{\sum_{j=1}^{t-1} x_{i,j}}{K_{i,t-1}}$ as the estimation of the expected reward without the manipulated variable $z_{i,t}$. By the definition of $\hat{\mu}_{i,t}$ and $\widetilde{\mu}_{i,t}$,
\begin{align}
& \mathbb{P}\left(|\widetilde{\mu}_{i,t-1} - \mu_i| > \sqrt{\frac{3\log{t}}{2K_{i,t-1}}} + \frac{\rho_{i,t-1}}{K_{i,t-1}}\right) \notag \\
    =\ & \mathbb{P}\left(\left|\hat{\mu}_{i,t-1} + \frac{\rho_{i,t-1}}{K_{i,t-1}}- \mu_i\right| > \sqrt{\frac{3\log{t}}{2K_{i,t-1}}} + \frac{\rho_{i,t-1}}{K_{i,t-1}} \ \right) \notag \\
    =\ & \mathbb{P}\left(\hat{\mu}_{i,t-1} - \mu_i > \sqrt{\frac{3\log{t}}{2K_{i,t-1}}}  \ \right) + \mathbb{P}\left(\hat{\mu}_{i,t-1} - \mu_i < -\sqrt{\frac{3\log{t}}{2K_{i,t-1}}} -2 \frac{\rho_{i,t-1}}{K_{i,t-1}}  \right) \notag \\
    \leq\ & \mathbb{P}\left(\hat{\mu}_{i,t-1} - \mu_i > \sqrt{\frac{3\log{t}}{2K_{i,t-1}}} \ \right) + \mathbb{P}\left(\hat{\mu}_{i,t-1} - \mu_i <  -\sqrt{\frac{3\log{t}}{2K_{i,t-1}}}\right) \notag \\
    =\ & \sum_{s=1}^{t-1} \mathbb{P}\left(|\hat{\mu}_{i,t-1} - \mu_i|> \sqrt{\frac{3\log{t}}{2s}}, K_{i,t-1} = s \ \right) \notag \\
    \leq\ & \sum_{s=1}^{t-1} \mathbb{P}\left(|\hat{\mu}_{i,{t-1}} - \mu_i|  > \sqrt{\frac{3\log{t}}{2s}}\ \right)\notag\\
    \leq\ & 2t\cdot \exp{(-3\log{t})} = \frac{2}{t^2} \,, \nonumber
\end{align}
where the last inequality follows the Chernoff-Hoeffding bound. By the union bound,
\[
\quad \mathbb{P}(\lnot E_t) = \mathbb{P}\left( \{   |\widetilde{\mu}_{i,t-1} - \mu_{i}| > \Lambda_{i,t}, \forall i \in [m]\} \right)\leq 2m \cdot t^{-2} \,. \qedhere
\]
\end{proof}

\begin{namedtheorem}[Theorem \ref{appendixA} (Re-statement)]
The regret of the strategic CUCB algorithm with m strategic arms in time horizon $T$ using an $(\alpha,\beta)$-approximation oracle is at most 
\begin{align}
 m\cdot\Delta_{max}\left( \frac{8B_{max} f^{-1}(\Delta_{min}) + 6\log{T}} {\left(f^{-1}(\Delta_{min})\right)^2} + \frac{\pi^2}{3}+1 \right)\,, \notag \nonumber
\end{align}
where $f^{-1}(\cdot)$ is the inverse bounded smoothness function.
\end{namedtheorem}

\begin{proof}
We start by introducing a few notations. Let $F_t$ to be the event where the $(\alpha,\beta)$-approximation oracle fails to produce an $\alpha$-approximation answer with respect to the input  $\bm{\Bar{\mu}} = (\Bar{\mu}_{1,t},\Bar{\mu}_{2,t},.....,\Bar{\mu}_{m,t})$ at time $t$. By definition of a $(\alpha,\beta)$-approximation oracle, we have $\mathbb{P}(F_t) \leq 1 - \beta$. Observe that at an arbitrary time $t$, a suboptimal arm subset may be selected due to two reasons, i) the $(\alpha, \beta)$-approximation oracle fails to provide an $\alpha$-approximate arm subset with respect to the input vector $\bm{\Bar{\mu}} = (\Bar{\mu}_{1,t},\Bar{\mu}_{2,t},.....,\Bar{\mu}_{m,t})$ and ii) the estimated mean vector $\bm{\bar{\mu}} = (\bar{\mu}_{1,t}, \bar{\mu}_{2,t},......,\bar{\mu}_{m,t})$ deviates from true values by a significant amount. 

To account for $\sum^T_{t=1} \mathbb{I}(S_t \in S_B)$,where $\mathbb{I}(S_t \in S_B)$ is 1 when algorithm choose a suboptimal arm subset, $S_t \in S_B$, at time $t$, we introduce a counter $N_i$ for each arm $i\in[m]$ after the $m$-round initialization and let $N_{i,t}$ be the value of $N_i$ at time $t$. We initialize $N_{i,m} = 1$. By definition, $\sum_{i \in [m]} N_{i,m} = m$. For $t > m$, the counter $N_{i,t}$ is updated as follows.
\begin{itemize}
    \item If $S_t \in S_B$, then $N_{{i'},t} = N_{{i'},t-1} + 1$ where $i' = \argmin_{i \in S_t} N_{i,t-1}$. In the case that $i'$ is not unique, we break ties arbitrarily.
    \item If $S_t \not \in S_B$, no counters will be updated then. 
\end{itemize}  
As such, the total number of pulls of suboptimal arm subsets is less than or equal to $\sum_{i=1}^m N_{i,T}$.

Define $\psi_t = \frac{8B_{max} f^{-1}(\Delta_{min}) + 6\log{t}} {\left(f^{-1}(\Delta_{min})\right)^2} > c\,, $
where $c$ is the larger solution of the following equation, 
\begin{align}
    &\frac{\left(f^{-1}(\Delta_{min})\right)^2}{4} c^2 
    - \left(2B_{max} f^{-1}(\Delta_{min}) + \frac{3\log{t}}{2}\right)c + 4B_{max}^2 = 0\,. \label{square}
\end{align}
By solving Equation \eqref{square}, we have  
\begin{align}
    c &= \frac{2\left(2B_{max} f^{-1}(\Delta_{min}) + \frac{3\log{t}}{2}\right)} {\left(f^{-1}(\Delta_{min})\right)^2} + \frac{2\sqrt{\left(2B_{max} f^{-1}(\Delta_{min}) + \frac{3\log{t}}{2}\right)^2 -4(f^{-1}(\Delta_{min})) B_{max}}} {\left(f^{-1}(\Delta_{min})\right)^2}\,. \label{psi}
\end{align}

We then decompose the total number $\sum_{i=1}^m N_{i,T}$ of pulls of suboptimal arm subsets as
\begin{align}
\quad\sum_{i=1}^m N_{i,T} &= m +\sum_{t=m+1}^T \mathbb{I}\{ S_t \in S_B\} \notag \\
&= m +\sum_{t=m+1}^T \sum_{i\in[m]}^m \mathbb{I}\{ S_t \in S_B, N_{i,t} > N_{i,t-1}, N_{i,t-1} \leq \psi_t\} + \sum_{t=m+1}^T \sum_{i\in[m]}^m \mathbb{I}\{ S_t \in S_B, N_{i,t} > N_{i,t-1}, N_{i,t-1} > \psi_t\}\notag \\
&\leq m + m\psi_T +\sum_{t=m+1}^T \sum_{i\in[m]}^m \mathbb{I}\{ S_t \in S_B, N_{i,t} > N_{i,t-1}, N_{i,t-1} > \psi_t\}\notag \\
&=m + m\psi_T+\sum_{t=m+1}^T \mathbb{I}\{  S_t \in S_B , \forall i \in S_t, N_{i,t-1} > \psi_t\} \,.\label{8}
\end{align}

The first inequality follows as $\sum_{t=m+1}^T \mathbb{I}\{ S_t \in S_B, \forall i \in S_t, N_{i,t-1} \leq \psi_t\}$ can be trivially upper bounded by $\psi_t$ and the second equality holds by our rule of updating the counters. 

The third term of Equation \eqref{8} can be further decomposed according to whether the oracle fails, 
\begin{align}
&\quad \sum_{t=m+1}^T \mathbb{I}\{ S_t \in S_B,  N_{i,t-1} > \psi_t, \forall i \in S_t\}\notag\\
&\leq \sum_{t=m+1}^T (\mathbb{I}\{F_t\} + \mathbb{I}\{ \lnot F_t, S_t \in S_B,  K_{i,t-1} > \psi_t, \forall i \in S_t\})\label{ke} \\
&= (T-m)(1-\beta)+\sum_{t=m+1}^T \mathbb{I}\{ \lnot F_t, S_t \in S_B, K_{i,t-1} > \psi_t , \forall i \in S_t\}\,. \notag
\end{align}

Let $\Lambda_{i,t} = \sqrt{\frac{3 \log{t} }{2K_{i,t-1}}} + \frac{\rho_{i,t-1}}{K_{i,t-1}}$ where $\rho_{i,t-1}$ is the total strategic budget spent by arm $i$ up to time $t-1$. Define event $E_t = \{  |\ \widetilde{\mu}_{i,t-1} - \mu_{i}| \leq \Lambda_{i,t}, \forall i \in [m]\}$, where $\mu_i$ is the true mean of arm $i$'s underlying distribution without manipulation. We continue the proof under $\mathbb{P}(\{ E_t, \lnot F_t, S_t \in S_B, K_{i,t-1} > \psi_t, \forall i \in S_t \}) = 0$ and prove it afterwards.

Since $\mathbb{P}\left(\{ E_t, \lnot F_t, S_t \in S_B, K_{i,t-1} > \psi_t , \forall i \in S_t\} \right) = 0$, by inclusion-exclusion principle, 
$ \mathbb{P}(\{\lnot F_t, S_t \in S_B, K_{i,t-1} > \psi_t, \forall i \in S_t\}) \leq \mathbb{P}(\lnot E_t)\,.$ Denote $\hat{\mu}_{i,t-1} = \frac{\sum_{j=1}^{t-1} x_{i,j}}{K_{i,t-1}}$ as the estimation of the expected reward without the manipulated variable $z_{i,t}$. By the definition of $\hat{\mu}_{i,t}$ and $\widetilde{\mu}_{i,t-1}$,
\begin{align}
&\quad \mathbb{P}\left(|\widetilde{\mu}_{i,t-1} - \mu_i| > \sqrt{\frac{3\log{t}}{2K_{i,t-1}}} + \frac{\rho_{i,t-1}}{K_{i,t-1}} +  \frac{B_{max}}{K_{i,t-1}}  \ \right) \notag \notag \\
    &= \mathbb{P}\left(\left|\hat{\mu}_{i,t-1} + \frac{\rho_{i,t-1}}{K_{i,t-1}}- \mu_i\right| > \sqrt{\frac{3\log{t}}{2K_{i,t-1}}} + \frac{\rho_{i,t-1}}{K_{i,t-1}}+\frac{B_{max}}{K_{i,t-1}} \right) \notag \\
    &\leq \mathbb{P}\left(\left|\hat{\mu}_{i,t-1} + \frac{\rho_{i,t-1}}{K_{i,t-1}}- \mu_i\right| > \sqrt{\frac{3\log{t}}{2K_{i,t-1}}} + \frac{\rho_{i,t-1}}{K_{i,t-1}} \ \right) \notag \\
    &\leq 2t^{-2}\,,
\end{align}
where the last inequality holds due to Lemma \ref{lambda}.

By the union bound, 
\begin{align}
    &\ \quad \mathbb{P}(\lnot E_t) = \mathbb{P}(  \{  |\widetilde{\mu}_{i,t-1} - \mu_{i}| > \Lambda_{i,t}, \forall i \in [m] \} \ ) \leq 2m \cdot t^{-2}\,. \notag
\end{align}
Hence, 
\begin{align}
     \mathbb{P}(\{\lnot F_t, S_t \in S_B,  \forall i \in S_t, K_{i,t-1} > \psi_t, \forall i \in S_t\}) \leq \mathbb{P}(\lnot E_t) \leq 2m \cdot t^{-2}\,. \notag
\end{align}

We now show that $\mathbb{P}(\{ E_t, \lnot F_t, S_t \in S_B, K_{i,t-1} > \psi_t, \forall i \in S_t \}) = 0$. Let $\Lambda = \sqrt{\frac{3 \log{t}}{2\psi_t}}+ \frac{2B_{max}}{\psi_t} $, which is not a random variable, and $B_{max} = max_{i \in [m]} B_i$, where $B_i$ is the strategic budget for arm $i$. For variable $\Lambda_{i,t}$, let $\Lambda_{i,t}^{\ast} = \max \{\Lambda_{i,t}\}$. Since $K_{i,t-1} > \psi_t$ and $B_{max} \geq B_i \geq \rho_i$, we have $\Lambda > \Lambda_{i,t}^{\ast}$.  According to line 7 of Algorithm 1, we have $\Bar{\mu}_{i,t}= \widetilde{\mu}_{i,t-1} + \sqrt{\frac{ 3 \log{t} }{2K_{i,t-1}}} + \frac{B_{max}}{K_{i,t-1}}$ and $\bar{\mu}_{i,t} = \hat{\mu}_{i,t-1} + \frac{\rho_{i,t-1}}{K_{i,t-1}} + \sqrt{\frac{ 3 \log{t} }{2K_{i,t-1}}}+ \frac{B_{max}}{K_{i,t-1}}$. Thus,  $|\hat{\mu}_{i,t-1} - \mu_i| \leq \Lambda_{i,t}$ implies that $0 < \bar{\mu}_{i,t-1} - \mu_i \leq 2\Lambda_{i,t} \leq 2\Lambda_{i,t}^{\ast} \leq 2\Lambda$, for $ i \in S_t$.

Recall $S^{\ast} = \argmax_{S \in \mathbb{S}} r_{\bm{\mu}}(S)$ and ${OPT}_{\bm{\mu}} = \max_{S\in \mathbb{S}} r_{\bm{\mu}}(S)$. Suppose $\{E_t, \lnot F_t, S_t \in S_B, \forall i \in S_t, K_{i,t-1} > \psi_t \} $ happens at time $t$, the following holds
\begin{align}
    &\quad r_{\bm{\mu}}(S_t) + f(2\Lambda) \notag \geq r_{\bm{\mu}}(S_t) + f(2\Lambda_{i,t}^{\ast})  \geq r_{\bm{\bar{\mu}}}(S_t) \geq \alpha \cdot {OPT}_{\bar{\bm{\mu}}}\geq \alpha \cdot r_{\bar{\bm{\mu}}}(S^{\ast}_{\bm{\mu}}) \geq \alpha \cdot r_{{\bm{\mu}}} (S^{\ast}_{\bm{\mu}})= \alpha \cdot {OPT}_{\bm{\mu}}\,. \notag
\end{align} 
The first inequality is due to the strict monotonicity of $f(\cdot)$ and $\Lambda > \Lambda_{i,t}^{\ast}$. The second inequality is due to the bounded smoothness property and $|\bar{\mu}_{i,t-1} - \mu_i | \leq 2\Lambda_{i,t}$. The third inequality is due to the fact that $\lnot F_t$ implies $S_t \geq \alpha \cdot opt_{\bm{\mu}}$. The forth inequality is by the definition of $opt_{\bm{\mu}}$. The last inequality is is due to the monotonicity of $r_{\bm{\mu}}(S)$ and $0 < \bar{\mu}_{i,t-1} - \mu_i$. 

Let $\kappa = \sqrt{\frac{3 \log{t}}{2c}} + \frac{2B_{max}}{c}$ where $c$ takes the value defined in Equation \eqref{square}. Given $\psi_t > c$, we have $\Lambda = \sqrt{\frac{3 \log{t}}{2\psi_t}} + \frac{2B_{max}}{\psi_t} < \kappa$. By Equation \eqref{psi}, we have $f(2\Lambda) < f(2\kappa) = \Delta_{min}$ and $\Delta_{min} > \alpha \cdot opt_{\bm{\mu}} -  r_{\bm{\mu}}(S_t)$, which contradicts the definition of $\Delta_{min}$ and the fact that $S_t \in S_B$. Therefore,
\begin{equation}
     \mathbb{P}(E_t, \lnot F_t, S_t \in S_B, \forall i \in S_t, K_{i,t-1} > \psi_t )= 0\,. \notag
\end{equation}
Based on the above analysis, we can upper bound the total number of suboptimal pulls as 
\begin{align}
     \mathbb{E}\left[\sum_{i=1}^m N_{i,T} \right] 
     &\leq m(1+\psi_T) + (T-m)(1-\beta) + \sum_{t=m+1}^T \frac{2m}{t^2}\notag \\
    &= m\left(1+\frac{8B_{max} f^{-1}(\Delta_{min}) + 6\log{T}} {\left(f^{-1}(\Delta_{min})\right)^2}\right) + (T-m)(1-\beta) + \sum_{t=m+1}^T \frac{2m}{t^2}\,.
\end{align}
Since the cumulative regret relate closely to the total number of suboptimal pulls $\mathbb{E}[\sum_{i=1}^m N_{i,T}]$, the upper bound of cumulative regret is thus
\begin{align}
    \quad Regret_{\bm{\mu},\alpha,\beta}(T)
    &\leq T \cdot \alpha  \beta  \text{OPT}_{\bm{\mu}} - \left( T \cdot \alpha  \text{OPT}_{\bm{\mu}} - \mathbb{E} \left[ \sum_{i=1}^m N_{i,T} \right] \Delta_{max} \right)
    \notag \\
     &=(\beta-1) T \cdot \alpha \text{OPT}_{\bm{\mu}} + \Delta_{max}\left( m \left(1+\frac{8B_{max} f^{-1}(\Delta_{min}) + 6\log{T}}  {\left(f^{-1}(\Delta_{min})\right)^2}\right)\notag + (T-m)  (1-\beta) +\sum_{t=m+1}^T \frac{2m}{t^2}\right)\notag\\
     &\leq ((T-m) \cdot \Delta_{max}- T \cdot \alpha \text{OPT}_{\bm{\mu}})(1-\beta) + m \cdot \Delta_{max}\left(1+\frac{8B_{max} f^{-1}(\Delta_{min}) + 6\log{T}} {\left(f^{-1}(\Delta_{min})\right)^2} + \frac{\pi^2}{3}\right)\,.\notag
\end{align}
Each time the algorithm pull a suboptimal arm subset $S_t \in S_B$ at time $t$, the algorithm incur an additional regret of at most $\Delta_{max}$, which is less than or equal to $\alpha \cdot opt_{\bm{\mu}}$. Thus,
\begin{align}
    &\quad (T-m) \Delta_{max}-T\alpha \cdot \text{OPT}_{\bm{\mu}} \notag\\
    &\leq (T-m) \alpha \cdot \text{OPT}_{\bm{\mu}}- T\alpha\cdot \text{OPT}_{\bm{\mu}}\notag\\
    &=-m  \alpha\cdot \text{OPT}_{\bm{\mu}} < 0\,.\notag
\end{align}
As a result, the regret of the strategic CUCB algorithm under strategic manipulations of reward is at most
\begin{align}
    &Regret_{\bm{\mu},\alpha,\beta}(T)\leq m \cdot \Delta_{max} \left(\frac{8B_{max} f^{-1}(\Delta_{min}) + 6\log{T}} {\left(f^{-1}(\Delta_{min})\right)^2} + \frac{\pi^2}{3} + 1\right)\notag \,.
\end{align}
\end{proof}

\section{Lower bound on the strategic budget}
\begin{namedtheorem}[Theorem \ref{lower}]
In stochastic multi-armed bandit problems, for a strategic suboptimal arm $i$, under time horizon $T$ and without access to other arms' information, the minimum strategic budget needed for it to be pulled $\omega(\log{T})$ is $\omega(\log{T})$. The regret incurred for any bandits learning algorithm is thus $\omega(\log{T})$.
\end{namedtheorem}
\begin{proof}
Let time $t \in [1,T]$ be the time step arm $i$ is last pulled under UCB algorithm and $\eta$ is a parameter chosen by the algorithm. The following inequality must stands at time $t$,
\begin{align*}
& \hat{\mu}_{i,t} + \sqrt{\frac{2\log(K_{i,t}^2/\eta^2)}{K_{i,t}}} + \frac{\rho_i}{K_{i,t}}\geq \hat{\mu}_{i^{\ast},t} + \sqrt{\frac{2\log(K_{i^{\ast},t}^2/\eta^2)}{K_{i^{\ast},t}}}\,.
\end{align*}
By Chernoff-Hoeffding bound and the union bound, 
\begin{align*}
&\quad \mathbb{P} \left(\hat{\mu}_{i,t}  - \mu_i  \geq \sqrt{\frac{2\log(K_{i,t}^2/\eta^2)}{K_{i,t}}}\right)\\
&\leq \sum^{t}_{s=1} \mathbb{P}  \left( \hat{\mu}_{i,t}  - \mu_i  \geq \sqrt{\frac{2\log(s^2/\eta^2)}{s}}, K_{i,t} = s \right) \\
&\leq\sum^{t}_{s=1} \mathbb{P}  \left( \hat{\mu}_{i,t}  - \mu_i  \geq \sqrt{\frac{2\log(s^2/\eta^2)}{s}}\right) \\
& \leq \sum^{t}_{s=1} \frac{\eta^2}{s^2} = \eta^2\sum^{t}_{s=1} \frac{1}{s^2} \leq  \frac{\pi^2}{6} \eta^2 \,.
\end{align*}
Thus $\hat{\mu}_{i,t}  - \mu_i  \leq \sqrt{\frac{2\log(K_{i,t}^2/\eta^2)}{K_{i,t}}}$ and similarly $\mu_{i^{\ast},t} -  \hat{\mu}_{i^{\ast}}  \leq \sqrt{\frac{2\log(K_{i^{\ast},t}^2/\eta^2)}{K_{i^{\ast},t}}} $, each with probability of at least $1 -  \frac{\pi^2}{6}\eta^2$. Hence with probability $1 - 2 \frac{\pi^2}{6}\eta^2$, we have 
\begin{align}
\mu_i + \sqrt{\frac{2\log(K_{i,t}^2/\eta^2)}{K_{i,t}}} + \frac{\rho_i}{K_{i,t}} \geq \mu_{i^{\ast}}\,. \notag
\end{align}
When arm $i$ is pulled, arm $i$ wants to ensure the following holds
\begin{align}
 \sqrt{\frac{2\log(K_{i,t}^2/\eta^2)}{K_{i,t}}} + \frac{\rho_i}{K_{i,t}} \geq \delta_i\,.\notag
\end{align}
where $\delta_i = \mu_{i^{\ast}} - \mu_i$. Then,
\begin{equation*}
B_i \geq \rho_i \geq \left(\delta_i - \sqrt{\frac{2\log(K_{i,t}^2/\eta^2)}{K_{i,t}}}\right) \cdot K_{i,t}\,. \qedhere
\end{equation*}
\end{proof}

\end{document}